\documentclass[11pt]{article}

\usepackage[a4paper,top=2.5cm,bottom=2.5cm,left=2.5cm,right=2.5cm,marginparwidth=2.5cm]{geometry}
\usepackage{lineno}
\usepackage{epsfig}
\usepackage{float}
\usepackage{multirow}
\usepackage{color}
\usepackage{lineno}
\usepackage{fullpage}
\usepackage[normalem]{ulem} 
\usepackage{makeidx}
\usepackage{xspace}
\usepackage{wrapfig}
\makeindex

\usepackage{dsfont}
\usepackage{xcolor}
\usepackage{graphicx}
\usepackage{caption}
\usepackage{subcaption}
\usepackage[sort]{cite}

\usepackage{amsmath,amsthm,amssymb,mathtools}

\newtheorem{theorem}{Theorem}
\newtheorem{lemma}[theorem]{Lemma}
\newtheorem{problem}[theorem]{Problem}
\newtheorem{remark}[theorem]{Remark}
\newtheorem{corollary}[theorem]{Corollary}

\usepackage{comment}

\usepackage{amsmath,amsfonts,bm}

\def\eqref#1{equation~(\ref{#1})}
\def\Eqref#1{Equation~(\ref{#1})}

\def\1{\bm{1}}

\DeclareMathAlphabet{\mathsfit}{\encodingdefault}{\sfdefault}{m}{sl}
\SetMathAlphabet{\mathsfit}{bold}{\encodingdefault}{\sfdefault}{bx}{n}

\def\gG{{\mathcal{G}}}

\def\gO{{\mathcal{O}}}

\newcommand{\E}{\mathbb{E}}

\makeatletter
\providecommand*{\cupdot}{%
  \mathbin{%
    \mathpalette\@cupdot{}%
  }%
}
\newcommand*{\@cupdot}[2]{%
  \ooalign{%
    $\m@th#1\cup$\cr
    \hidewidth$\m@th#1\cdot$\hidewidth
  }%
}
\makeatother

\begin{document}

%

%

\title{A Simple Proof of the Universality of\\Invariant/Equivariant Graph Neural Networks}

\author{
{Takanori Maehara\footnote{RIKEN AIP, email: takanori.maehara@riken.jp}} \and {Hoang NT\footnote{RIKEN AIP, Tokyo Institute of Technology, email: hoang.nguyen.rh@riken.jp}}
}

\maketitle

\begin{abstract}
\noindent
   We present a simple proof for the universality of invariant and equivariant tensorized graph neural networks. Our approach considers a restricted intermediate hypothetical model named Graph Homomorphism Model to reach the universality conclusions including an open case for higher-order output. We find that our proposed technique not only leads to simple proofs of the universality properties but also gives a natural explanation for the tensorization of the previously studied models. Finally, we give some remarks on the connection between our model and the continuous representation of graphs.
\end{abstract}

\section{Introduction}

\subsection{Background and Motivation}

In this study, we consider a \emph{graph regression problem}.
Let $\mathcal{G}$ be the set of simple directed graphs with vertex and edge weights.
\begin{problem}
We are given pairs $(G_1, y_1), (G_2, y_2), \dots$ of input graphs $G_i \in \mathcal{G}$ and outcomes $y_i \in \mathcal{Y}$.
The task is to learn a hypothesis $h \colon \mathcal{G} \to \mathcal{Y}$ such that $h(G_i) \approx y_i$.
\end{problem}
This problem naturally arises in practice.
For example, in a toxicity detection problem~\cite{zitnik2017predicting}, we want to learn a function $h \colon \mathcal{G} \to \mathcal{Y}$ ($ = \mathbb{R}$) such that $h(G) = 1$ if $G$ has a toxicity and $h(G) = 0$ otherwise.
For another example, in a community detection problem~\cite{bruna2017community}, we want to learn a function $h \colon \mathcal{G} \to \mathcal{Y}$ ($= \mathbb{R}^{(V(G))_2}$) such that $h(G)_{uv} = 1$ if the vertices $u$ and $v$ are in the same community and $h(G)_{uv} = 0$ otherwise.
Here, we denote by $(V)_k = \{ (x_1, \dots, x_k) \in V^k : x_i \neq x_j \ (i \neq j) \}$ for a finite set $V$ and a positive integer $k$.%
\footnote{We only consider the values of equivariant functions at pairwise different indices. The general case is easily handled by considering each pattern separately, but this complicates the notation. Thus we concentrate on this case.}

We are often interested in a hypothesis that merely depends on the topology of the graph (see Section~\ref{sec:related}).
Mathematically, this condition is represented by invariance and equivariance.
A function $h \colon \mathcal{G} \to \mathcal{Y}$ ($ = \mathbb{R}$) is \emph{invariant} if for any graph $G = (V(G), E(G))$ and a permutation $\sigma$ of $V(G)$, the following equation holds:
\begin{align}
    h(G^\sigma) = h(G),
\end{align}
where $G^\sigma$ is the graph whose indices of vertices are permuted by $\sigma$, i.e., $G^\sigma = (V(G), E(G)^\sigma)$ with $E(G)^\sigma = \{ (\pi(u), \pi(v)) : (u, v) \in E(G) \}$.
A function $h \colon \mathcal{G} \to \mathbb{R}^{(V(G))_k}$ is \emph{equivariant} if for any graph $G = (V(G), E(G))$ and any permutation $\sigma$ on $V(G)$, the following equation holds:
\begin{align}
    h(G^\sigma) = h(G)^\sigma,
\end{align}
where $h(G)^\sigma$ is defined by the relation $h(G)_{x_1,\dots,x_k} = h(G)_{\pi(x_1) \dots \pi(x_k)}^\pi$ for all $\vec{x} = (x_1, \dots, x_k) \in (V(G))_k$.
The invariance and equivariance mean that the output of the function $h$ is determined up to isomorphism.


One desirable property of a hypothesis space is the \emph{universal approximation property} (\emph{universality} for short)~\cite{cybenko1989approximation,hornik1991approximation}, i.e., for any continuous function is arbitrary accurately by a function in the hypothesis space.
Maron et al.~\cite{maron2019universality} introduced a feed-forward invariant neural network and proved that this model has the universality in the continuous invariant functions.
They also characterized all the invariant linear layers~\cite{maron2018invariant}.
Keriven and Peyr\'{e}~\cite{keriven2019universal} extended the tensorized graph neural network (Figure~\ref{fig:tensorized_model}) to represent equivariant functions and proved the universality for the $k = 1$ dimensional output case.
Then, they left the universality for the higher-order output case as an open problem.

\begin{figure}
    \centering
    \begin{subfigure}[b]{0.48\textwidth}
        \centering
        \includegraphics[width=\textwidth]{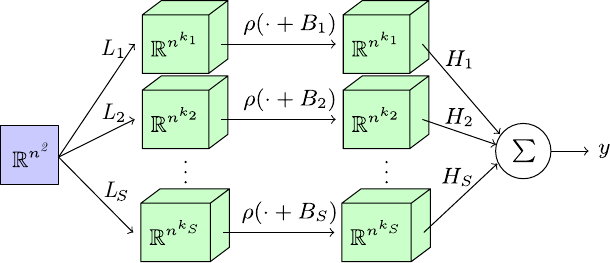}
        \caption{Tensorized Graph Neural Network \cite{keriven2019universal}}
        \label{fig:tensorized_model}
    \end{subfigure}
    \begin{subfigure}[b]{0.48\textwidth}
        \centering
        \includegraphics[width=0.78\textwidth]{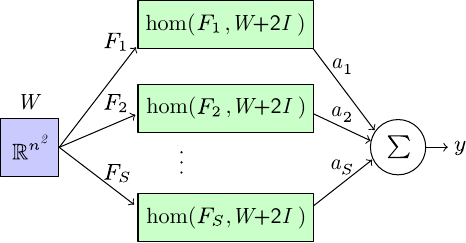}
        \caption{Graph Homomorphism Model}
        \label{fig:homo_model}
    \end{subfigure}%
    \caption{A visual comparison between Tensorized Graph Neural Network \cite{keriven2019universal} and our Graph Homomorphism  Model. In Figure~\ref{fig:tensorized_model}, we denote the equivariant linear operators by $L_i: \mathbb{R}^{n^2} \mapsto \mathbb{R}^{n^{k_i}} $ to avoid confusion with our notation of $F_i$ as graph. We also restrict the input of the Tensorized Graph Neural Network to graphs ($\mathbb{R}^{n^2}$) in place of hypergraphs.}
    \label{fig:models}
\end{figure}

\subsection{Contribution}

In this study, we give a simple proof of the universality of tensorized neural networks for both invariant and (higher-order) equivariant cases; 
the latter solves an open problem posed in \cite{keriven2019universal}.
Our proof relies on a result in graph theory (see Section~\ref{sec:compare} for a comparison of proof techniques in the existing studies).

Let $\mathcal{F}$ be the set of simple unweighted graphs, and let $F = (V(F), E(F)) \in \mathcal{F}$.
Let $G = (V(G), E(G); W)$ be a weighted graph, where $W \colon V(G) \times V(G) \to \mathbb{R}$ is the weighted adjacency matrix.
Let $\mathcal{S}(V)$ be the set of all permutations on $V$.
Then, the \emph{homomorphism number} is defined by
\begin{align}
    \mathrm{hom}(F, W) 
    = \hspace{-2.0em} \sum_{\pi \colon V(F) \to V(G)} & \prod_{i \in V(F)} W(\pi(i), \pi(i)) \nonumber \\
    \times \hspace{-1.0em} & \prod_{(i,j) \in E(F)} W(\pi(i), \pi(j)).
\end{align}
Similarly, for a given $\vec{x} \in (V(G))_k$, the \emph{$k$-labeled homomorphism number} is defined by
\begin{align}
\label{eq:labeled-homomorphism-density}
    \mathrm{hom}_{\vec{x}}(F, W) 
    = \hspace{-2.0em} \sum_{\substack{\pi \colon V(F) \to V(G) \\ \pi(i) = x_i \ (i \in [k])}} & \prod_{i \in V(F)} W(\pi(i), \pi(i)) \nonumber \\
    \times \hspace{-1.0em} & \prod_{(i,j) \in E(F)} W(\pi(i), \pi(j)).
\end{align}
By the definition, homomorphism numbers and $k$-labeled homomorphism numbers are continuous functions in $W$ that is invariant and equivariant, respectively.

Let $\mathcal{G}_0$ be the set of weighted directed graphs whose edge weights are bounded by one and the number of vertices is $n$.
Let $\mathcal{A}$ and $\mathcal{A}'$ be the set of functions of the following forms
\begin{align}
    \mathcal{A} &= \left\{ W \mapsto \hspace{-0.3em} \sum_{F \in \mathcal{F}}^{\text{finite}} a_F \mathrm{hom}(F, W + 2 I)  : a_F \in \mathbb{R} \right\}, \\
    \mathcal{A}' &= \left\{ W \mapsto \hspace{-0.3em} \sum_{F \in \mathcal{F}}^{\text{finite}} a_F \mathrm{hom}_{\vec{x}}(F, W + 2 I) : a_F \in \mathbb{R} \right\},
\end{align}
where $I$ denotes the identity matrix of $n \times n$.
We prove the following theorems.
\begin{theorem}
\label{thm:universality-invariance}
$\mathcal{A}$ is dense in the continuous invariant functions.
\end{theorem}
\begin{theorem}
\label{thm:universality-equivariance}
$\mathcal{A}'$ is dense in the continuous equivariant functions.
\end{theorem}
The translated homomorphism number $W \mapsto \mathrm{hom}(F, W + 2 I)$ and the translated $k$-labeled homomorphism number $W \mapsto \mathrm{hom}_{\vec{x}}(F, W + 2 I)$ are invariant and equivariant linear function on the $|E(F)|$-fold tensor product, respectively.
Therefore, we can implement them in a tensorized neural network.
This means that our model is less powerful than the tensorized graph neural network.
On the other hand, because our models have the universality, we obtain the following.
\begin{corollary}
The invariant (resp., equivariant) tensorized graph neural network has the universality in continuous invariant (resp., equivariant) functions.
\qed
\end{corollary}

%

\subsection{Related Work}
\label{sec:related}

In practice, the design of a machine learning model (e.g. neural networks) usually follows some prior knowledge about the target functions since restriction bias helps to simplify the learning process. For instance, in image processing, convolutional neural networks \cite{lecun1989backpropagation} are designed to be translation invariant \cite{krizhevsky2012imagenet} or shift-invariant \cite{zhang2019making}. Therefore, much research has been conducted to address the universality of general invariant neural networks. More recently, graph neural networks predicting labels of vertices \cite{chebynets,gcn,dgi} have hinted the importance of equivariant models. A natural question in learning theory is whether these aforementioned models are \emph{universal}? \cite{hornik1991approximation,maron2018invariant}. Here, we discuss related work that answered this question.

\paragraph{Invariant models} The invariant property of a model is usually discussed in the context of learning from points clouds and sets~\cite{qi2017cvpr,zaheer2017deep,sannai2019universal}, then generalized to symmetries~\cite{maron2018invariant} and group actions~\cite{cohen2016group}. While universality analyses for models on sets are well developed, the analysis for graphs is limited~\cite{maron2018invariant,maron2019universality,kondor2018generalization}. Recently, Maron et al.~\cite{maron2019universality} proved a neural network which is $G$-invariant is universal. Similarly, Keriven and Peyr\'{e}~\cite{keriven2019universal} obtained the universal result on tensorized graph neural networks by using a more direct application of the Stone-Weierstrass theorem.

\paragraph{Equivariant models} The existence of equivariant models only makes practical sense in some specific cases, for example, learning on graphs' vertices. Therefore, comparing to the invariant case, there are only a limited number of work addressing equivariance \cite{ravanbakhsh2017equivariance,keriven2019universal,sannai2019universal}. Consequently, the universality of equivariant graph models have only has recently proven by Keriven and Peyr\'{e}~\cite{keriven2019universal}. 


\section{Proofs}

To prove the universality of a class of functions, we use the Stone--Weierstrass theorem:
\begin{theorem}[{Stone--Weierstrass Theorem~\cite[Theorem 1.1]{nel1968theorems}}]
\label{thm:stone-weierstrass}
Let $\mathcal{X}$ be a compact Hausdorff space and $C(\mathcal{X}, \mathbb{R})$ be the set of continuous functions from $\mathcal{X}$ to $\mathbb{R}$, equipped with the $\infty$-norm.
If a subalgebra $\mathcal{A} \subseteq C(\mathcal{X}, \mathbb{R})$
satisfies the following two conditions: 
\begin{itemize}
    \item $\mathcal{A}$ separates points, i.e., for any $x \neq y$, there exists $h \in \mathcal{A}$ such that $h(x) \neq h(y)$, and
    \item There exists $u \in \mathcal{A}$ that is bounded away from zero, i.e., $\inf_{x \in \mathcal{X}} |u(x)| > 0$,
\end{itemize}
then $\mathcal{A}$ is dense in $C(\mathcal{X}, \mathbb{R})$.
\qed
\end{theorem}
The proofs are devoted to verify the conditions of the Stone--Weierstrass theorem.

\subsection{Proof of Theorem~\ref{thm:universality-invariance} (Invariant Case)}

We first define the graph space.
Let $n$ be the number of vertices in input graphs.
Let 
\begin{align}
    \mathcal{G}_0 = \left\{ \right. & W \in \mathbb{R}^{n \times n} : \left. |W(i, j)| \le 1, \ \forall i, j \right\}
\end{align}
be the set of the weighted adjacency matrices.
We denote by $[n] = \{1, \dots, n\}$.
The $l_1$ norm\footnote{Because $\mathcal{G}_0$ is a finite-dimensional vector space, any norms are equivalent; thus, the result in this section is invariant with respect to the choice of the norm.} of the graphs is given by
\begin{align}
    \| W \|_1 = \sum_{u, v \in [n]} |W(u, v)|.
\end{align}
Then, we introduce the \emph{edit distance} $\delta_1$ by 
\begin{align}
    \delta_1(W_1, W_2) = \min_{\sigma \in \mathcal{S}([n])} \| W_1 - W_2^\sigma \|_1,
\end{align}
where $\mathcal{S}([n])$ is the set of all permutations on $[n]$.
The edit distance $\delta_1$ is nonnegative and satisfies the triangle inequality, i.e., it is a pseudo-metric.
We define the \emph{graph space} by the metric identification as $\tilde{\mathcal{G}}_0 = \mathcal{G}_0 / {\sim}$ where $W_1 \sim W_2$ if and only if $\delta_1(W_1, W_2) = 0$.
This forms a metric space.

Any invariant function $f \colon \mathcal{G}_0 \to \mathbb{R}$ is identified as a function $f \colon \tilde{\mathcal{G}}_0 \to \mathbb{R}$.
Our goal is to prove the universality for the set $C(\tilde{\mathcal{G}}_0, \mathbb{R})$ of continuous functions from $\tilde{\mathcal{G}}_0$ to $\mathbb{R}$.
Now, we check the conditions of the Stone--Weierstrass theorem.

First, we check the condition of the space.
\begin{lemma}
\label{lem:compactness-of-graph-space}
The graph space $(\tilde{\mathcal{G}}_0, \delta_1)$ is a compact Hausdorff space.
\end{lemma}
\begin{proof}
It is Hausdorff because it is a metric space.
We show the sequential compactness.
Let $W_1, W_2, \dots$ be an arbitrary sequence in $\tilde{\mathcal{G}}_0$, which is also identified as a sequence in $\tilde{\mathcal{G}}_0$.
Because $\mathcal{G}_0$ is compact in the $l_1$ norm, we can choose a convergent subsequence. 
Such sequence is also a convergent subsequence in $\tilde{\mathcal{G}}_0$.
Thus, $\tilde{\mathcal{G}}_0$ is compact.
\end{proof}

Next, we check the conditions of $\mathcal{A}$.
\begin{lemma}
\label{lem:algebra-invariant}
$\mathcal{A}$ forms an algebra.
\end{lemma}
\begin{proof}
Clearly, it is closed under the addition and the scalar multiplication.
It is closed under the product because of the following identity:
\begin{align}
\label{eq:algebra-invariant}
    \mathrm{hom}(F_1, W) \cdot \mathrm{hom}(F_2, W) = \mathrm{hom}(F_1 \cupdot F_2, W),
\end{align}
where $F_1 \cupdot F_2$ is the disjoint union of $F_1$ and $F_2$.
\end{proof}
\begin{lemma}
\label{lem:bounded-away-from-zero-invariant}
$\mathcal{A}$ contains an element that is bounded away from zero.
\end{lemma}
\begin{proof}
Let $\circ$ be the singleton graph.
Then, $W \mapsto \mathrm{hom}(\circ, W + 2I) \ge n$ is bounded away from zero.
\end{proof}

To prove the separate point property, we use the following theorem.
\begin{theorem}[{\cite[Lemma~2.4]{lovasz2006rank}, $k = 0$ case, in our terminology}]
\label{thm:lovasz2006rank-invariant}
Let $W_1, W_2 \in \mathbb{R}^{n \times n}$ be matrices with positive diagonal elements.
Then, $W_1$ and $W_2$ are isomorphic if and only if $\mathrm{hom}(F, W_1) = \mathrm{hom}(F, W_2)$ for all simple unweighted graph $F$.
\qed
\end{theorem}

\begin{lemma}
\label{lem:separate-points-invariant}
$\mathcal{A}$ separates points in $\tilde{\mathcal{G}}_0$.
\end{lemma}
\begin{proof}
If $W_1, W_2 \in \mathcal{G}_0$ are non-isomorphic, then $W_1 + 2 I$ and $W_2 + 2 I$ are also non-isomorphic.
Because $W_1 + 2 I$ and $W_2 + 2 I$ satisfy the condition in Theorem~\ref{thm:lovasz2006rank-invariant}, there exists $F$ such that $\mathrm{hom}(F, W_1 + 2 I) \neq \mathrm{hom}(W_2 + 2 I)$.
This means that $\mathcal{A}$ separates points in $\tilde{\mathcal{G}}_0$.
\end{proof}

Therefore, we proved Theorem~\ref{thm:universality-invariance}.

\subsection{Proof of Theorem~\ref{thm:universality-equivariance} (Equivariant Case)}

We identify an array-valued function $h \colon \mathcal{G}_0 \to \mathbb{R}^{([n])_k}$ as a two-argument function $h' \colon \mathcal{G}_0 \times ([n])_k \to \mathbb{R}$. 
Let $\mathcal{G}_0' = \mathcal{G}_0 \times ([n])_k$.
Then, each element $(W, \vec{x}) \in \mathcal{G}_0'$ is identified as a $k$-labeled graph, which is a graph with $k$ distinguished vertices $\vec{x} = (x_1, \dots, x_k) \in (V(G))_k$.

For a permutation $\sigma \in \mathcal{S}([n])$, we define
\begin{align}
    (W, \vec{x})^\sigma := (W^\sigma, \vec{x}^\sigma),
\end{align}
where $\vec{x}^\sigma = (x_1, \dots, x_k)^\sigma := (x_1^\sigma, \dots, x_k^\sigma)$. 
Then, $h$ is equivariant if and only if $h'$ is invariant in the sense that $h'((W, \vec{x})^\sigma) = h'((W, \vec{x}))$.
We say that $(W_1, \vec{x}_1)$ and $(W_2, \vec{x}_2)$ are \emph{isomorphic} if $\delta_1'((W_1, \vec{x}_1), (W_2, \vec{x}_2)) = 0$.

Now we define the \emph{$k$-labeled edit distance} $\delta_1'$ by 
\begin{align}
\label{eq:labeled-l1-distance}
    \delta_1'((W_1, \vec{x}_1), (W_2, \vec{x}_2)) 
    := \min_{\substack{\sigma \in \mathcal{S}([n]) \\ \vec{x}_1 = \vec{x}_2^\sigma}} \| W_1 - W_2^\sigma \|_1.
\end{align}
Then, we define the \emph{$k$-labeled graph space} by the metric identification as $\tilde{\mathcal{G}}_0' = \mathcal{G}_0' / {\sim}$ where $(W_1, \vec{x}_1) \sim (W_2, \vec{x}_2$ if and only if $\delta_1((W_1, \vec{x}_1), (W_2, \vec{x}_2)) = 0$, i.e., these are isomorphic.
This forms a metric space.


Any equivariant function $f \colon \mathcal{G}_0 \to \mathbb{R}^{([n])_k}$ is identified as a function $f \colon \tilde{\mathcal{G}}_0' \to \mathbb{R}$.
Our goal is to prove the universality for the set $C(\tilde{\mathcal{G}}_0', \mathbb{R})$ of continuous functions from $\tilde{\mathcal{G}}_0'$ to $\mathbb{R}$.
Now, we check the condition of the Stone--Weierstrass theorem.
This part is very similar to that of the invariant case.

First, we check the condition of the space.
\begin{lemma} 
\label{lem:compactness-of-labeled-graph-space}
The $k$-labeled graph space $(\tilde{\mathcal{G}}_0', \delta_1)$ is compact.
\end{lemma}
\begin{proof}
It is Hausdorff because it is a metric space.
We show the sequentially compactness.
Let $(W_1, \vec{x}_1), (W_2, \vec{x}_2), \dots$ be an arbitrary sequence in $\tilde{\mathcal{G}}_0'$, which is also identified as a sequence in $\mathcal{G}_0'$.
Because the number of possibilities of $\vec{x}_i$ is finite, we can select an (infinite) subsequence that has the same value of $\vec{x}_i$.
The remaining part is the same as the proof of Lemma~\ref{lem:compactness-of-graph-space}.
\end{proof}

Next, we check the conditions of $\mathcal{A}'$.
\begin{lemma}
\label{lem:algebra-equivariant}
$\mathcal{A}'$ forms an algebra.
\end{lemma}
\begin{proof}
Clearly, it is closed under the addition and the scalar multiplication.
It is closed under the product because of the following identity.
\begin{align}
    \mathrm{hom}_{\vec{x}}(F_1, W) \mathrm{hom}_{\vec{x}}(F_2, W) = \mathrm{hom}_{\vec{x}}(F_1 \cupdot' F_2, W).
\end{align}
where $F_1 \cup' F_2$ is the graph obtained from the disjoint union of $F_1$ and $F_2$ by glueing the labeled vertices.
\end{proof}

\begin{lemma}
\label{lem:bounded-away-from-zero-equivariant}
$\mathcal{A}'$ contains an element that is bounded away from zero.
\end{lemma}
\begin{proof}
Let $\circ \cdots \circ$ be the graph of $k$ isolated vertices.
Then, $W \mapsto \mathrm{hom}_{\vec{x}}(\circ \cdots \circ, W + 2 I) \ge n - k$ is bounded away from zero.
\end{proof}

To prove the separate point property, we use the following theorem.
\begin{theorem}[{\cite[Lemma~2.4]{lovasz2006rank} in our terminology}]
\label{thm:lovasz2006rank}
Let $W_1, W_2 \in \mathbb{R}^{n \times n}$ be matrices with positive diagonal elements.
Let $\vec{x}_1, \vec{x}_2 \in ([n])_k$.
Then, $(W_1, \vec{x}_1)$ and $(W_2, \vec{x}_2)$ are isomorphic if and only if $\mathrm{hom}_{\vec{x}_1}(F, W_1) = \mathrm{hom}_{\vec{x}_2}(F, W_2)$ for all simple unweighted graph $F$.
\qed
\end{theorem}

\begin{lemma}
$\mathcal{A}'$ separates points in $\tilde{\mathcal{G}}_0'$.
\end{lemma}
\begin{proof}
It is proved similarly as Lemma~\ref{lem:separate-points-invariant} by using Theorem~\ref{thm:lovasz2006rank}.
\end{proof}

Therefore, we proved Theorem~\ref{thm:universality-equivariance}.

\section{Comparison with Other Proofs}
\label{sec:compare}

\paragraph{Compare with Keriven and Peyr\'{e} \normalfont \cite{keriven2019universal}.}

Our proofs are similar to that of their proofs.
For the invariant case, they used the standard Stone--Weierstrass theorem and verified the separable points property by constructing functions on higher-order tensor space. 
For the equivariant case, they developed a new Stone--Weierstrass type theorem, and verified the corresponding separate point properties by a similar technique to the invariant case.
On the other hand, for both cases, we used the standard Stone--Weierstrass theorem, and verified the separate point property using the property of the homomorphism number.
This unified treatment allows us to establish the result on arbitrary higher-order outputs.

One advantage of their method is that it is applicable to hypergraphs. 
Our method could be applicable for hypergraphs; however, there is a gap because the theory of weighted homomorphism number of hypergraphs is not well established (compared with graphs).

Note that they considered graphs with different but bounded numbers of vertices (i.e., $|V(G_i)| \le n_\text{max}$).
However, this is not effective because such space is disconnected, and each connected components corresponds to the graphs having the same number of vertices.
%
If we have to consider a set of graphs with different numbers of vertices, it is promising to consider graphons; see Section~\ref{sec:graphon}.


\paragraph{Compare with Maron et al. \normalfont \cite{maron2018invariant}}

They only considered the invariant case.
They used the universality of symmetric polynomials by Yarotsky~\cite{yarotsky2018universal}.
Then, they approximated the polynomials by a tensorized neural network.

One advantage of their method is that one can bound the order of tensors.
Our method can also bound the order of tensors by bounding the size of the subgraphs $F$ \cite[Theorem~5.33]{lovasz2012large}; however, this may give a loose bound.
On the other hand, our method shows a very restricted form of the linear invariant (or equivariant) layers are sufficient to obtain the universality.



\section{Concluding Remarks}

\label{sec:graphon}

In this study (and the existing studies~\cite{maron2018invariant,keriven2019universal}), the number $n$ of vertices in the input graphs are fixed.
This is reasonable because the graph space is disconnected and each connected component corresponds to the graphs of the same number of vertices;
hence, a continuous function in the graph space of different numbers of vertices is just a collection of continuous functions in each connected component.

If we want to consider graphs of different numbers of vertices, it is promising to consider \emph{graphons}~\cite{lovasz2012large}.
Below we explain that all the results obtained in this paper can be extended to graphons.

An (asymmetric) graphon is a measurable function $[0, 1]^2 \to [0, 1]$.
This is a continuous generalization of the weighted adjacency matrix.
The set of graphons is denoted by $\mathcal{W}_0$.
The \emph{cut-norm} $\| \cdot \|_\square$ is defined by
\begin{align}
    \| W \|_\square = \sup_{S,T \subseteq [0,1]} \left| \int_{S \times T} W(x,y) dx dy \right|.
\end{align}
and the \emph{cut-distance} $\delta_\square$ is defined by
\begin{align}
    \delta_\square(W_1, W_2) = \inf_{\sigma} \| W_1 - W_2^\sigma \|_\square,
\end{align}
where $\sigma$ runs over all measure-preserving bijections.
The \emph{graphon space} $\tilde{\mathcal{W}}_0$ is defined by the metric identification.

The graphon space contains infinitely many large graphs.
However, it is still compact with respect to the cut distance.
Note that this does not hold for the edit distance.
\begin{theorem}[{\cite[Theorem~9.23]{lovasz2012large}}]
The graphon space $(\tilde{\mathcal{W}}_0, \delta_\square)$ is compact.
\end{theorem}

In graphons, we use the homomorphism density instead of the homomorphism number.
Let $F = (V(F), E(F))$ be a simple unweighted graph.
Then, the \emph{homomorphism density} is given by
\begin{align}
    t(F, W) = \int \prod_{(i,j) \in E(F)} W(x_i, x_j) \prod_{i \in V(F)} dx_i.
\end{align}
The set of linear combination of the homomorphism densites also forms an unital algebra.
To show the separate points property, we can use the following theorem.
\begin{theorem}[{Directed version of \cite[Corollary~10.34]{lovasz2012large}}]
Let $W_1, W_2 \in \tilde{\mathcal{W}}_0$.
$W_1$ and $W_2$ are isomorphic if and only if $t(F, W_1) = t(F, W_2)$ for all simple unweighted graph $F$.
\qed
\end{theorem}

Therefore, we obtain the following result.
\begin{theorem}
The set of finite linear combinations of the homomorphism densities are dense in the continuous invariant graphon functions.
\qed
\end{theorem}
Note that this fact is already proved in \cite[Theorem 17.6]{lovasz2012large} for a different context (for symmetric graphons).

The equivariant case is also handled by considering $k$-labeled graphons.
The \emph{$k$-labeled homomorphism density} is given by
\begin{align}
    t_{\vec{x}}(F, W) = \int \prod_{(i,j) \in E(F)} W(x_i, x_j) \prod_{i \in V(F) \setminus [k]} dx_i.
\end{align}
Then, we obtain the following result.
\begin{theorem}
The set of finite linear combinations of the $k$-labeled homomorphism densities are dense in the continuous equivariant graphon functions.
\qed
\end{theorem}

\bibliographystyle{plain}
\bibliography{graphons}

\end{document}